\documentclass[letterpaper, 10 pt, conference]{ieeeconf}  

\IEEEoverridecommandlockouts                              

\usepackage{graphics} 
\usepackage{amsmath} 
\usepackage{amssymb}  

\title{\LARGE \bf
Variance-Based Risk Estimations in Markov Processes via Transformation with State Lumping
}

\author{Shuai Ma and Jia Yuan Yu
\thanks{Shuai Ma and Jia Yuan Yu are with the Concordia Institute of Information System Engineering, Concordia University, Montreal, Canada. Email:	
{\tt\small m\_shua@encs.concordia.ca, jiayuan.yu@concordia.ca}.}%
} 

\usepackage{dsfont,algorithm,algorithmic,graphicx,caption,xcolor} 

\usepackage{stackengine,subcaption,cleveref,mathtools}

\newtheorem{theorem}{Theorem}

\newtheorem{corollary}{Corollary}
\newtheorem{definition}{Definition}

\crefname{equation}{}{}
\newcommand{\rom}[1]{\uppercase\expandafter{\romannumeral #1\relax}}

\begin{document}

\maketitle
\thispagestyle{empty}
\pagestyle{empty}

\begin{abstract}

Variance plays a crucial role in risk-sensitive reinforcement learning, and most risk measures can be analyzed via variance.
In this paper, we consider two law-invariant risks as examples: mean-variance risk and exponential utility risk.
With the aid of the state-augmentation transformation (SAT), we show that, the two risks can be estimated in Markov decision processes (MDPs) with a stochastic transition-based reward and a randomized policy.
To relieve the enlarged state space, a novel definition of isotopic states is proposed for state lumping, considering the special structure of the transformed transition probability.
In the numerical experiment, we illustrate state lumping in the SAT, errors from a naive reward simplification, and the validity of the SAT for the two risk estimations. 

\end{abstract}


\section{INTRODUCTION}

In reinforcement learning (RL), the standard criterion is the expectation of (discounted) total reward.
However, stronger reliability guarantees are required in many problems, especially where the small probability events have serious consequences, such as self-driving cars and medical diagnosis.
In these cases, a risk-sensitive criterion should be considered.
In sequential decision making problems, a risk-sensitive criterion refers to a risk measure, or a risk function, which assigns a scalar to a reward sequence $ \{ R_t \} $.
In this paper, we focus on law-invariant risk measures~\cite{kusuoka2001law}, which are functions mapping a set of random variables to the real numbers.
In an infinite time-horizon Markov decision process (MDP), a law-invariant risk measure is usually on the return $\sum_{t=1}^{\infty} \gamma^{t-1} R_t $, where $ \gamma \in (0,1) $ is the discount factor.
Two risk measures are considered as examples: mean-variance risk and exponential utility risk.

We show that, with the return variance formula and the state-augmentation transformation (SAT), the two risks can be estimated in an MDP with a stochastic transition-based reward and a randomized policy.
The variance can be calculated in a Markov process, but with a deterministic reward function only.
Besides, a number of RL methods have similar requirements~\cite{borkar2002q,garcia2015comprehensive,huang2017risk,chow2017risk,berkenkamp2017safe} even in a risk-sensitive scenario. 
However, the reward functions are usually stochastic in many practical problems.
It has been shown that, when the objective is risk-sensitive, and the reward needs to be converted to a simple form, the SAT should be implemented instead of a reward simplification~\cite{shuai2019satAAAI}.
In this paper, we present the SAT in a homomorphism version, and thoroughly discuss its pros and cons.
A numerical experiment shows how an SAT preserved the variance in an MDP with a randomized policy, comparing with the reward simplification.

The contribution of this paper is twofold.
Firstly, we extend the SAT with a state lumping theorem to relieve the enlarged state space. Secondly, we estimate the two law-invariant risk measures with the return variance.
In Section~\rom{2}, we review different types of internal risk measures in MDPs, and provide two variance-based estimation formulas.
In Section~\rom{3}, we give the MDP notations, restate the SAT as a homomorphism of MDP. 
In Section~\rom{4}, we propose a novel definition of isotopic states for state lumping in transformed Markov process, illustrate the SAT in a toy MDP example, and show the error on the two risks from naive reward simplification.
In Section~\rom{5}, we have the conclusions and discuss the future research.

\section{RISK MEASURES}
\label{secRisk}
In RL, uncertainty is studied from two perspectives. 
One is the \textit{external} uncertainty, which refers to the parameter uncertainty or disturbance.
When the model is unknown, its parameters are usually estimated first, and then a model-based approach can be implemented to achieve the optimal policy.
However, the parameter estimation depends on noisy data in practice, and the modeling errors may result in negative consequences.
In control theory, this problem is known as robust control.
Robust control methods consider the uncertain parameters within some compact sets, and optimize the expected return with the worst-case parameters, in order to achieve good robust performance and stability~\cite{nilim2005robust}.

The other risk concerns the \textit{inherent} (or \textit{internal}) uncertainty, which results from the stochastic nature of the process.
The inherent risk can be quantified by a dynamic measure~\cite{ruszczynski2010risk} or a law-invariant measure~\cite{kusuoka2001law}. 
The immediate reward at epoch $ t $ is defined by $ R_t $, a dynamic measure can be denoted in general as
\[ 
R_1 + \rho_1(R_2 + \rho_2(R_3 +\rho_3(\cdots))), 
\]
which is sensitive to the order of the immediate rewards.
Dynamic measures are usually assumed to have a set of properties, such as Markov, monotonicity and coherence, which yields a time-consistent risk measure with a nested structure.
In this paper, we focus on law-invariant risk measures.
Given a discount factor $ \gamma $, a law-invariant measure in an infinite horizon is a functional $ \Psi $ on the return $ \Phi = \sum_{t=1}^{N} \gamma^{t-1} R_t $, where $ N \in \mathbb{N} \cup \{+\infty\} $.
Three types of law-invariant risk have been widely studied in RL area. 

\textit{Mean-variance risk}: 
The mean-variance risk is also known in finance as the modern portfolio theory~\cite{D.J.White1988a,doi:10.1287/opre.42.1.175,Mannor2011a}.
The mean-variance analysis aims at optimal return at a given level of risk, or the optimal risk at a given level of return.
In RL, several mean-variance models have been studied.
The variance and the standard deviation of the return are denoted by $ \mathbb{V}(\Phi) $ and $ \sigma (\Phi) $, respectively.
One model could be
\begin{equation} \label{eq:meanVar}
\Psi_V(\Phi) = \mathbb{E}(\Phi) - k \sigma (\Phi),
\end{equation} 
where $ k $ is a risk parameter, and when $ k > 0 $, it is a risk-averse objective.
This is the first mean-variance model for exploring inventory management related problems~\cite{lau1980newsboy}. 
The other model can be maximizing the expected return with a variance constraint, or minimizing the variance with an expected return constraint~\cite{choi2008mean}.
For a review on mean-variance risk, see~\cite{chiu2016supply}.

\textit{Utility risk}: 
Another type of risk measure can be utility risks. 
The original goal of a utility function is to represent the subjective preference~\cite{howard1972risk}. 
One classic example can be the ``St. Petersburg Paradox,'' which refers to a lottery with an infinite expected reward, but no one would put up an arbitrary high stake to play it, since the probability of obtaining a high enough reward is too small. 
Mathematically, a utility function $ U: \mathbb{R} \rightarrow \mathbb{R} $ is a mapping from objective value space for all possible outcomes to subjective value space. 
A utility objective is usually in the form $ U^{-1}\{\mathbb{E}[U(\Phi)]\} $, 
where $ U $ is a strictly increasing function.
The most common used utility risk in RL is the exponential utility~\cite{chung1987discounted}
\[
\Psi_U(\Phi) = \beta^{-1}\log\{\mathbb{E}[\exp(\beta \Phi)]\},
\]
where $ \beta $ models a constant risk sensitivity that risk-averse when $ \beta < 0$.
This can be seen more clearly with the Taylor expansion of the utility
\begin{equation} \label{eq:utility}
\beta^{-1}\log\{\mathbb{E}[\exp(\beta \Phi)]\} = \mathbb{E}(\Phi) + \frac{\beta}{2}\mathbb{V}(\Phi) + \mathcal{O}(\beta^2).
\end{equation}
\textit{Quantile-based risk}:
The last type of risk measure used in practice refers to quantiles, which requires us to pay attention to discontinuities and intervals of quantile numbers.
A commonly used quantile-based risk measure is value at risk (VaR). 
For VaR estimation with the SAT, see~\cite{shuai2019satAAAI}.
In short, we claim that most, if not all, inherent risk measures depend on the reward sequence $ (R_t:t \in \{ 1, \cdots, N \}) $, 
which can be preserved by the SAT in a risk-sensitive scenario.
For law-invariant risk estimations, since VaR estimation in a Markov process with a stochastic reward has been throughly studied in~\cite{shuai2019satAAAI}, in this paper, we focus on the first two risk estimations with Eq.~\cref{eq:meanVar,eq:utility}, respectively.
In a MDP with a randomized policy, the return variance $ \mathbb{V}(\Phi) $ is required by both of the two risk estimations.
With the aid of the SAT, $ \mathbb{V}(\Phi) $ can be calculated in a Markov process with a stochastic reward, which allows risk evaluation in practical RL problems.

\section{STATE-AUGMENTATION TRANSFORMATION}

In this section, firstly, we present notations of MDPs with a deterministic and stochastic reward. 
Secondly, we restate the state-augmentation transformation (SAT) as an MDP homomorphism. 
Thirdly, we have a discussion on the pros and cons of the SAT in different cases.

\subsection{Markov Decision Processes}

In this paper, we focus on infinite-horizon discrete-time MDPs with finite states and actions. 
An MDP with a deterministic reward can be represented by
\[
\mathcal{M}^{\dagger} = \langle S, A, r, p, \mu , \gamma \rangle,
\]
in which
$S $ is a finite state space, and $X_t \in S$ represents the state at (decision) epoch $t \in \mathbb{N}$;
$A_x$ is the allowable action set for $x \in S$, $A = \bigcup_{x \in S}A_x$ is a finite action space, and $K_t \in A$ represents the action at epoch $t$; 
$r$ is a bounded reward function, and $R_t$ denotes the immediate reward at epoch $t$;
$p(y \mid x, a) = \mathbb{P}(X_{t+1}=y \mid X_t = x, K_t = a)$ denotes the homogeneous transition probability;
$\mu$ is the initial state distribution; $ \gamma \in (0,1)$ is the discount factor.
Since a number of risk-sensitive RL studies require the reward function to be deterministic and state-based~\cite{shuai2019satAAAI}, we only consider $ r: S \rightarrow \mathbb{R} $ in a transformed MDP (or Markov process).

Similarly, an MDP with a stochastic reward can be represented by
\[
\mathcal{M} = \langle S, A, J, p, d, \mu , \gamma \rangle,
\]
in which
$J$ is a finite subset of $ \mathbb{R} $, and is the set of possible values of the immediate rewards;
$ d $ is the reward distribution, with $d(j \mid x, a, y) = \mathbb{P}(R_t=j \mid X_t = x, K_t = a, X_{t+1}=y)$ the probability that the immediate reward at time $ t $ is $ j $, given current state $ x $, action $ a $, and next state $ y $.

A policy $\pi$ refers to a sequence of decision rules ($\pi_0, \pi_1, \cdots,\pi_{N}$) for $ N \in \mathbb{N} \cup \{+\infty\} $, which describes how to choose actions sequentially. 
A randomized decision rule is $ \pi_i: S \rightarrow \mathbb{D}(A) $, i.e., for a given state, it outputs a distribution on the action space. 
Randomized policy is often considered in constrained MDPs~\cite{altman1999constrained} or MDPs with a risk objective. 
For an MDP with an infinite time horizon, we consider a stationary policy, which can be considered as one decision rule, i.e., $ \pi: S \rightarrow \mathbb{D}(A) $.
A Markov reward process is tantamount to an MDP with a randomized policy.
A Markov process with a deterministic reward function can be represented by
\[
\mathcal{M}^{\dagger}_{\pi} = \langle S, r_{\pi}, p_{\pi}, \mu , \gamma \rangle.
\]
In this study, we consider an MDP with a randomized policy.
The function of SAT is to transform $ \mathcal{M}^{\dagger}  $ to $ \mathcal{M} $ and preserve $ (R_t) $, and this problem is concerned in four cases~\cite{shuai2019satAAAI}. %
To illustrate how the SAT with state lumping works, we consider the Case 3 as an example, which refers to an MDP with a randomized policy.

\subsection{Homomorphism Version of the Transformation}
\label{secHomo}
Here, we restate the SAT as an MDP homomorphism, which is more general from a mathematical perspective.
In many practical problems, the rewards of the Markov processes are stochastic, but many methods may require the reward in a simple form.
To enable the method in a risk-sensitive case, we may use the SAT to transform the Markov process and preserve the reward sequence $ (R_t:t \in \{ 1, \cdots, N \}) $.
In a sense, the SAT generalizes a number of theoretical studies which are originally designed to work for MDPs or Markov processes with deterministic reward only. 
In this paper, the example is that the variance formula is for Markov processes with deterministic reward functions only.
We restate the SAT as an MDP homomorphism.
Comparing with the original SAT theorem~\cite{shuai2019satAAAI}, the homomorphism version of SAT works on a more abstract level.
An MDP homomorphism is a formalism that captures an intuitive notion of specific equivalence between MDPs~\cite{ravindran2002model}. %
In order to convert an MDP $ \mathcal{M} $ with $ d $ to $ \mathcal{M}^{\dagger} $ with $ r $ and preserve $ (R_t) $, we consider each ``situation'', which determines immediate reward, as an augmented state.
We can then attach each possible reward value to an augmented state in $ \mathcal{M}^{\dagger} $.
Formally, we define the SAT homomorphism as follows.

\begin{definition} [SAT, a homomorphism version]
	
	The SAT for MDPs is a homomorphism $ h $ from an MDP $ \mathcal{M} = \langle S, A, J, p, d, \mu \rangle $ to an MDP $ \mathcal{M}^{\dagger} = \langle S^{\dagger}, A, r, p^{\dagger}, \mu^{\dagger} \rangle $. 
	The state space $ S^{\dagger} = S^{\ddagger} \cup S_n $, where 
	$S^{\ddagger} = S^2 \times A \times J$, 
	$ S_n = \{s_{null,x}\}_{x \in S} $, 
	and $ S_n \cap S^{\ddagger} = \emptyset $.
	For $ x^{\dagger} = (x,a,y,i), y^{\dagger} = (y,a_y,z,j) \in S^{\ddagger} $, we have $ A_{x^{\dagger}} = A_y $, $ r(x^{\dagger}) = i $, and for $ a_y \in A_y $, $ p^{\dagger}(y^{\dagger} \mid (x^{\dagger}, a_y) = p^{\dagger}(y^{\dagger} \mid (s_{null,y}, a_y) = p(z \mid y,a_y) d(j \mid y,a_y,z) $; 
	for $ x^{\dagger} = s_{null,x} \in S_n $, we have $ A_{x^{\dagger}} = A_x $, $ r(x^{\dagger}) = 0 $, and $ \mu^{\dagger}(s_{null,x}) = \mu(x) $.
	
\end{definition}
We call $ \mathcal{M}^{\dagger} $ the homomorphic image of $ \mathcal{M} $ under $ h $.
For any policy $ \pi $ in $ \mathcal{M} $, there exists a policy $ \pi^{\dagger} $ in $ \mathcal{M}^{\dagger} $, such that the two processes share the same $ (R_t) $.
We define the mapping between the two policy spaces as a policy lift.

\begin{definition} [Policy lift]
	Let $ \mathcal{M}^{\dagger} $ be a homomorphic image of $ \mathcal{M} $ under $ h $.
	Let $ \pi $ be a stochastic policy in $ \mathcal{M} $.
	Then $ \pi $ lifted to $ \mathcal{M}^{\dagger} $ is the policy $ \pi^{\dagger} $ such that $ \pi^{\dagger}(a \mid x^{\dagger}) = \pi(a \mid y)$ for $ x^{\dagger} = (x,a,y,i) \in S^{\ddagger}$ and $ \pi^{\dagger}(a \mid x^{\dagger}) = \pi(a \mid x)$ for $ x^{\dagger} = s_{null,x} \in S_n$.
\end{definition}

Given an MDP with a policy, the randomness of the induced Markov reward process can be studied in its underlying probability space. 

\begin{definition} [Underlying probability space]
	Let $ (\Omega, \mathcal{F}, \mathcal{P}) $ be a probability space, and $ (E, \mathcal{B}) $ a measurable space with $ E = S \times J $.
	An induced Markov reward process can be represented by an $ (E, \mathcal{B}) $-valued stochastic process on $ (\Omega, \mathcal{F}, \mathcal{P}) $ with  a family $ (Y_t)_{t \in \mathbb{N}} $ of random variables $ Y_t: (\Omega, \mathcal{F}) \rightarrow (E, \mathcal{B}) $ for $ t \in \mathbb{N} $.
	$ (\Omega, \mathcal{F}, \mathcal{P}) $ is called the underlying probability space of the process $ (Y_t)_{t \in \mathbb{N}} $. 
	For all $ \omega \in \Omega $, the mapping $ Y(\cdot, \omega): t \in \mathbb{N} \rightarrow Y_t(\omega) \in E $ is called the trajectory of the process with respect to $ \omega $.
	The process $ (Y_t)_{t \in \mathbb{N}} $ is progressively measurable with respect to the filtration $ (\mathcal{F}_t)_{t \in \mathbb{N}} $. 
\end{definition}

A homomorphism version of the SAT theorem is as follows, which claims that the probability measure on trajectories is preserved under $ h $.
Therefore, as a subsequence of sample path, the probability measure on $ (R_t) $ is preserved as well.

\begin{theorem} [Probability measure preservation]
	\label{satUsed}
	Let $ \mathcal{M}^{\dagger} $ be an image of $ \mathcal{M} $ under homomorphism $ h $.
	Let $ \pi^{\dagger} $ be the stochastic policy lifted from $ \pi $.
	For the two processes $ \mathcal{M} $ with $ \pi $ and $ \mathcal{M}^{\dagger} $ with $ \pi^{\dagger} $, there exists a bijection $ f_{\Omega}: \Omega \rightarrow \Omega^{\dagger} $, such that
	for the underlying probability space $ (\Omega, \mathcal{F}, \mathcal{P}) $ for the first process, we have a sample path probability space $ (\Omega^{\dagger}, \{ f_{\Omega}(b): b \in \mathcal{F} \}, P^{\dagger}) $ for the second process, such that 
	for any $ t \in \mathbb{N} $, $ \mathcal{P}^{\dagger}(\{ f_{\Omega}(b): b \in \mathcal{F}_t \}) = \mathcal{P}(\{\mathcal{F}_t\})$.
\end{theorem}

The SAT theorem claims that, for an $\mathcal{M}$ with a stochastic transition-based reward, there exists an $\mathcal{M}^{\dagger}$ with a deterministic state-based reward, such that for any given $ \pi $ for $\mathcal{M}$, there exists a corresponding $ \pi^{\dagger} $ for $\mathcal{M}^{\dagger}$, such that both Markov reward processes share the same $ (R_t) $.  

\textit{Pros and cons of the SAT}: There are mainly two benefits brought by the SAT.
Firstly, it converts an $ \mathcal{M} $ (or $ \mathcal{M}_{\pi} $), which has a discrete stochastic reward distributed on a finite support, into an $ \mathcal{M}^{\dagger} $ (or $ \mathcal{M}^{\dagger}_{\pi} $) with a deterministic state-based reward function.
The essence of the SAT is a surjective mapping from the situation space~\cite{shuai2019satAAAI} to the reward space $ J $.
This mapping not only extends a number of methods provided for $ \mathcal{M}^{\dagger} $ to work for $ \mathcal{M} $, but also allows us to analyze the immediate rewards as states with methods such as z-transform~\cite{howard1964dynamic}.
Secondly, action is ``removed'' from the reward function and becomes a component of the augmented state.
This renders action to impact on transitions only.
There are two disadvantages as well.
One is the size of the augmented state space is now ($ |S|^2|A||J| $), which is much larger than $ |S| $. 
The other is that the SAT removes the recurrence property from a recurrent MDP, which may result in a prohibition of some methods.
In the next section, we apply the SAT to an MDP with a randomized policy, estimate the two risks, and compare the results with the estimations from reward simplification.

\section{RISK ESTIMATION VIA TRANSFORMATION}
In this section, we use a toy example to illustrate the SAT with a state lumping aiming at shrinking the augmented state space.
Furthermore, we show the risk estimations and errors from reward simplification.
Firstly, we describe the MDP with two states and a randomized policy, and propose a novel definition of isotopic class, in which all contained states can be regarded as one state in return study. 
Secondly, we estimate the two risks with the return variance calculation method, and show how the reward simplification change the risk estimations.

\begin{figure}[t] 
	\centering
	\includegraphics[width=0.45\textwidth]{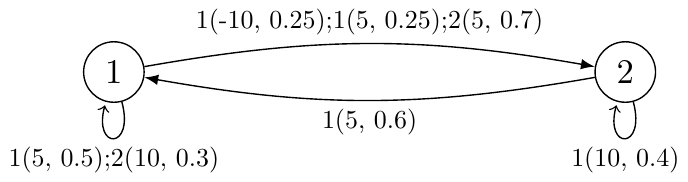} 
	\caption{A toy example with two states and two actions. The labels $ a(j,q) $ along transitions represent the action $ a $, the immediate reward $ j $, and the probability $ q = p(y \mid x, a) d(j \mid x, a, y) $.} 
	\label{MdpDesc}
\end{figure}

\subsection{Isotopic State lumping in a Toy Example} 

Consider the process illustrated in Fig.~\ref{MdpDesc} with two states and two actions, in which state 1 has two actions and state 1 has one. 
The process starts from state 1 at time $ t = 1 $.  
Let's consider the randomized policy $ \pi(1) = [0.5, 0.5] $---uniformly choose an action for state 1---then we have a Markov reward process illustrated in Fig. 2(a). 
In order to calculate the return variance, the Markov process needs to have a deterministic reward function.
One way to achieve this is to naively simplify the reward function by calculating the expectation conditioned on state.
In the MDP with $ \pi $, the reward simplification refers to
\begin{equation} \label{eq:Rsim}
r'(x) = \sum\limits_{a \in A_x, y \in S, j \in J} j \cdot \pi(a \mid x) p(y \mid x,a) d(j \mid x,a,y),  
\end{equation}
for $ x \in S $.
The induced Markov process is shown in Fig. 2(b).

The other way to acquire a deterministic reward function is through SAT.
A suitable SAT renders a deterministic reward function and preserve $  (R_t)  $.
However, SAT also enlarges the state space from $ |S| $ to ($ |S|^2 |A| |J| $) in general.
Considering the special structure of $ p_{\pi}^{\dagger} $, we propose a definition of isotopic states as an alleviation of the enlarged state space.

\begin{definition} [Isotopic states]
	\label{isoState}
In a Markov process $ \langle S, r_{\pi}, p_{\pi}, \mu , \gamma \rangle  $, if there exist two states $ x_i,x_j \in S$, such that,
\begin{itemize}
\item Condition 1: $ r_{\pi}(x_i) = r_{\pi}(x_j) $; and 
\item Condition 2: $ p_{\pi}(y \mid x_i) = p_{\pi}(y \mid x_j) $ for $ y \in S\setminus \{x_i,x_j\} $, 
\end{itemize}
then we say $ x_i$ and $x_j$ are isotopic. 
\end{definition}
With the definition of isotopic states, we propose a theorem on reward preservation in state lumping as follows.
\begin{figure}[!b]
	\centering	
	\begin{subfigure}{\columnwidth} %
		\includegraphics[width=\columnwidth]{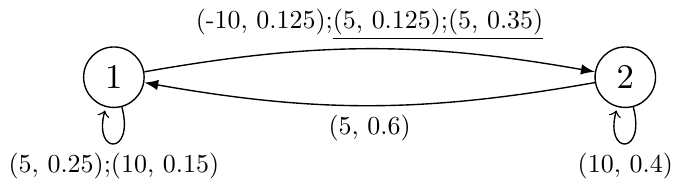}
		\caption{A Markov process with a stochastic transition-based reward function. The labels $ (j,q) $ along transitions represent the immediate reward $ j $ and the probability $ q = p_{\pi}(y \mid x) d_{\pi}(j \mid x, y) $, where $ x,y $ represent the current and the next states, respectively. The two underlined situations are isotopic states, which can be lumped into one augmented state in the transformed Markov process.}\label{fig:mrp}
	\end{subfigure}
	\vspace{10pt}
	\\	
	\begin{subfigure}{\columnwidth} %
		\centering
		\includegraphics[width=0.85\columnwidth]{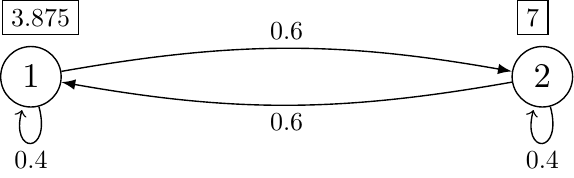}
		\caption{The Markov process from reward simplification. The labels $ q $ along transitions represent the probability $ q = p_{\pi}(y \mid x) $, and the labels $ r(x) $ in the text boxes represent the deterministic state-based reward values from reward simplification.}\label{fig:mc_Rsim}
	\end{subfigure}
	\caption{The reward simplification on the Markov process with a stochastic transition-based reward function.}
	
\end{figure}

\begin{theorem} [Reward preservation in state lumping]
	\label{isotopic}
	If the two states $ x_i,x_j \in S$ are isotopic in a Markov process $ \mathcal{M} = \langle S, r_{\pi}, p_{\pi}, \mu , \gamma \rangle  $, then there exists a Markov process $ \mathcal{M}' = \langle S', r'_{\pi}, p'_{\pi}, \mu' , \gamma \rangle  $, in which $ S' = S \setminus \{x_j\} $, and for $ x,y \in S' \setminus \{x_i\} $ and $ z \in S' $, $ r'_{\pi}(z) = r_{\pi}(z) $, $ p'_{\pi}(x_i \mid y) = p_{\pi}(x_i \mid y) + p_{\pi}(x_j \mid y) $, $ p'_{\pi}(x \mid y) = p_{\pi}(x \mid y) $, $ p'_{\pi}(x \mid x_i) = p_{\pi}(x \mid x_i) $, $ p'_{\pi}(x_i \mid x_i) = p_{\pi}(x_i \mid x_i) + p_{\pi}(x_j \mid x_i) $, and $ \mu'(y) = \mu(y) $, $ \mu'(x_i) = \mu(x_i) + \mu(x_j) $,
	such that the two Markov process $ \mathcal{M} $ and $ \mathcal{M}' $ share the same $ (R_t:t \in \{ 1, \cdots, N \}) $.
\end{theorem}

\begin{proof}
	Denote the reward sequence for $ \mathcal{M}' $ by $ (R'_t:t \in \{ 1, \cdots, N \}) $.
	Since for $ x \in S'\setminus \{x_i\}$, $ \mu'(x) = \mu(x) $, $ \mu'(x_i) = \mu(x_i) + \mu(x_j) $, we have $ R_1 $ and $ R'_1 $ share the same distribution.
	Since the outcome $ x_j $ is replaced by $ x_i $ in $ \mathcal{M}' $, the two events $ X_t \in \{x_i, x_j\} $ and $X'_t = x_i $ share the same probability conditioned on $ X_{t-1} $; 
	and since $ p(y \mid x_i) = p(y \mid x_j) $ for $ y \in S\setminus \{x_i,x_j\} $, the replacement does not change the probability of $X'_{t+1} = x \in S' $ conditioned on $ X_{t} $.
	Then, for $ t = 2, \cdots, N $, $ R_t $ and $ R'_t $ share the same probability conditioned on $ X_{t-1} $.
	Therefore, the two Markov process $ \mathcal{M} $ and $ \mathcal{M}' $ share the same $ (R_t:t \in \{ 1, \cdots, N \}) $.
\end{proof}

The theorem above claims that, the isotopic state lumping will not change the reward sequence. Therefore, most, if not all, risk measures will be preserved by an induced Markov process with a smaller state space from the isotopic state lumping. 
Furthermore, we call a set of states isotopic class if any two members in the set are isotopic.
It is easy to generalize the Theorem~\ref{isotopic} to an isotopic class.
\begin{corollary} [Reward preservation in class lumping]
	In a transformed Markov process $ \langle S, r_{\pi}, p_{\pi}, \mu_{\pi} , \gamma \rangle  $, if there exist two states $ x_i,x_j \in S$, which satisfies two conditions: i).  $ r_{\pi}(x_i) = r_{\pi}(x_j) $; and ii). $ p(y \mid x_i) = p(y \mid x_j) $ for $ y \in S\setminus \{x_i,x_j\} $, then we say $ x_i$ and $x_j$ are isotopic.
\end{corollary}

\begin{figure}[!b] 
	\centering
	\includegraphics[width=\columnwidth]{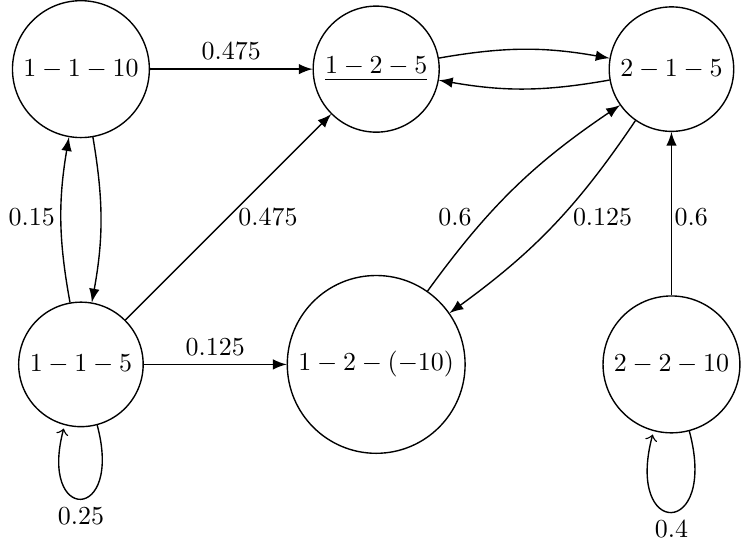} 
	\caption{The transformed Markov process with a deterministic state-based reward function.
		Some transitions are hidden.
		The underlined state comes from lumping two situations from the original Markov process in Fig. 2(a).} 
	\label{Mc_trans}
\end{figure}

The results of this section are based in part on the results of~\cite{kemeny1976markov,burke1958markovian}, in which the lumpability in Markov processes was thoroughly studied.
By comparison, the Definition~\ref{isoState} and the Theorem~\ref{isotopic} are based on single states instead of partitions of state space, i.e., in Condition 2 in Definition~\ref{isoState}, $ x_i, x_j $ and $ y $ are single states instead of partitions.
Besides the fact that it is hard to partition a general state space with equivalent partitioned transition probabilities (Definition 8.4 in~\cite{harrison1992performance}), the main reason refers to a characteristic of the transformed transition probability from the SAT.
Let's consider the transformed Markov process in Fig.~\ref{Mc_trans} with a state space $ S $ as an example.
Since the transition distribution $ p_{\pi}(\cdot \mid x) $ for $ x = (x^{\dagger}, y^{\dagger}, j) \in S$ depends on $ y^{\dagger} $ only, 
a sufficient condition for two states being isotopic could be that, the two states share the last two components.
This sufficient condition can be stated as follows.
\begin{corollary} [A sufficient condition for isotopic states]
	In a transformed Markov process with a state space $ S $, if there exist two states $ x = (x_1,x_2,i), y = (y_1, y_2,j) \in S$, such that $ x_2 = y_2 $ and $ i = j $, then $ x$ and $y$ are isotopic.
\end{corollary}
Now, we can clearly see in Fig.~\ref{Mc_trans} that, the two states $ (1-1-5) $ and $ (2-1-5) $ are isotopic as well, which can be lumped to further simplify the analysis.
In summary, for states which satisfied the two conditions in Definition~\ref{isoState}, we may regard them as one state and keep $ (R_t)$ intact.
Considering that the transition probability in the transformed Markov process is only sensitive to the start and end original states, the augmented state space can be shrunk under a fair condition. 
The transformed Markov process is illustrated in Fig.~\ref{Mc_trans} with some transitions hidden.
Notice that the underlined state is generated by lumping the two situations shown in Fig. 2(a). 
Though the isotopic class lumping is not specially designated for SAT, the transformed Markov process has a fair chance to relieve from the enlarged state space to some degree.
In the next section, we estimate the two risks on the return with the aid of the SAT.




\begin{figure}[t] 
	\centering
	\includegraphics[width=\columnwidth]{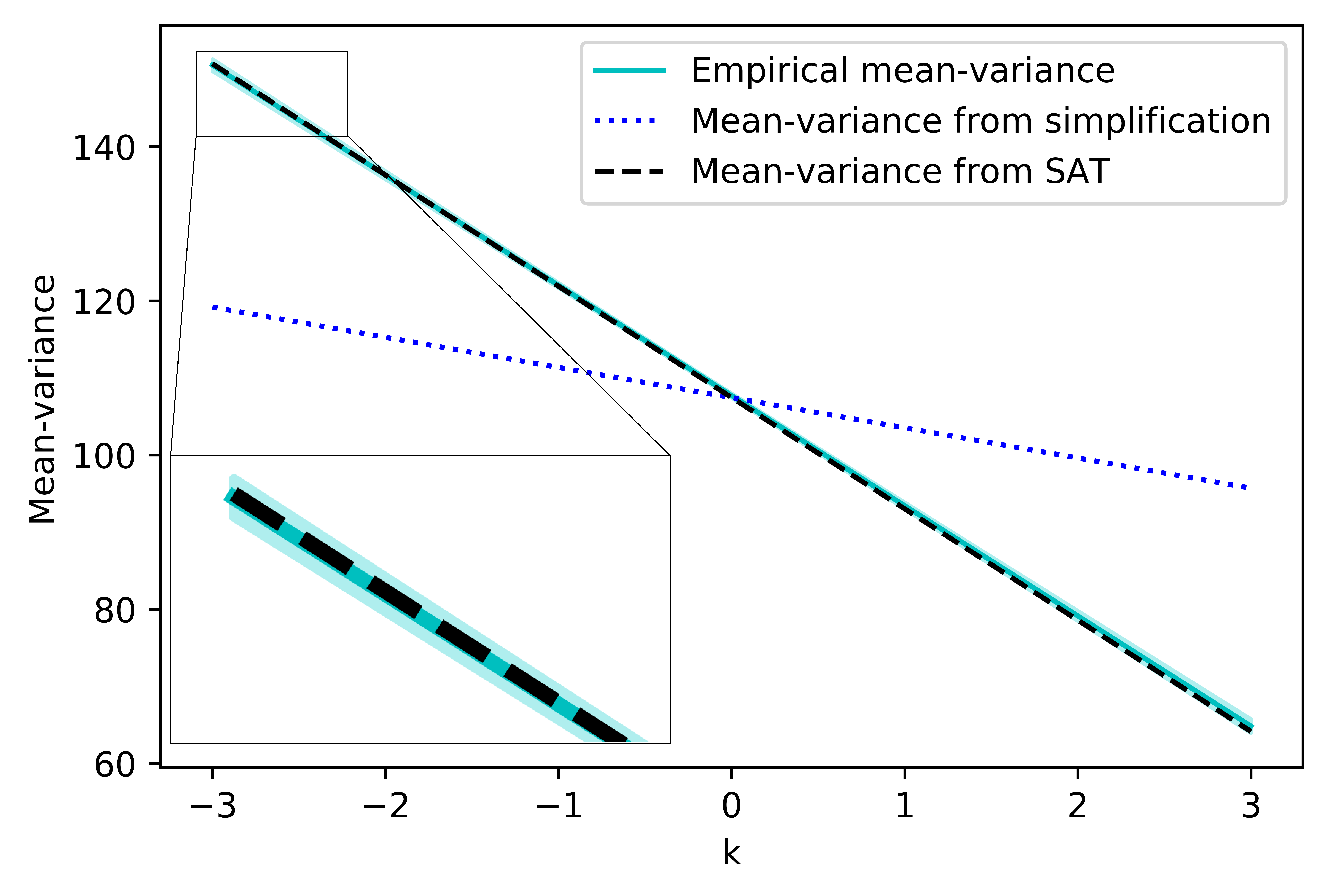} 
	\caption{The comparison among the empirical mean-variance risk (with a narrow error region zoomed in), the estimated mean-variance risk from reward simplification and the estimated mean-variance risk from SAT along the risk parameter $ k \in (-3,3) $.} 
	\label{meanVar}
\end{figure}

\subsection{Risk Estimations and Errors from Reward Simplification}
Here, we estimate the mean-variance risk and the exponential utility risk by Eq.~\cref{eq:meanVar,eq:utility}, respectively.
As shown in Section~\ref{secRisk}, the two risks can be estimated with the return variance.
For an infinite-horizon Markov reward process with a deterministic reward function, Sobel presented the formula for the return variance.
See~\cite{sobel1982variance} for further information.
%
Notice that the variance formula is for Markov processes with deterministic reward functions only.
Though we have the SAT for Markov process with a discrete stochastic reward, how to apply the method to the ones with a continuous stochastic reward is a concern. 

For each risk, we compare three estimations: empirical estimation, estimation from reward simplification, and estimation from SAT.
The empirical estimation is calculated as follows. 
We run $ L = 50 $ groups of simulations to calculate the variance of an estimation, in each group we run $ M=1000 $ simulations of the Markov process, and in each simulation we set the time horizon $ N = 1000 $.

\textit{Mean-variance risk estimation}: 
The empirical estimation of mean-variance risk with a risk parameter $ k $ is 
\[
\hat{\Psi}_V (\beta) = \sum_{i=1}^{L}\Psi_{V,i}(\beta) / L,
\]
in which 
\[
\Psi_{V,i}(\beta) = \sum_{t=1}^{M} \phi_{i,t} /M,
\]
where $ \phi_{i,t} $ is an outcome of a simulation in group $ i $, and $ t \in \{1, \cdots, M\} $.
The comparison among the empirical mean-variance risk, the estimated mean-variance risk from reward simplification, and the estimated mean-variance risk from SAT along the risk parameter $ k \in (-3,3) $ is shown in Fig.~\ref{meanVar}.
For different $ k $, the empirical mean-variance risk $ \hat{\Psi}_V (\beta) $ is illustrated with an error region representing the standard deviations of the means.
Since the error region is so narrow that it is hardly seen, we zoom in a piece to make it clear.
Based on the observation, we can see that, the estimated mean-variance risk from SAT is close to the empirical mean-variance risk, but its counterpart from reward simplification is not.
That is because the return variance is preserved by the SAT.
\begin{figure}[t] 
	\centering
	\includegraphics[width=\columnwidth]{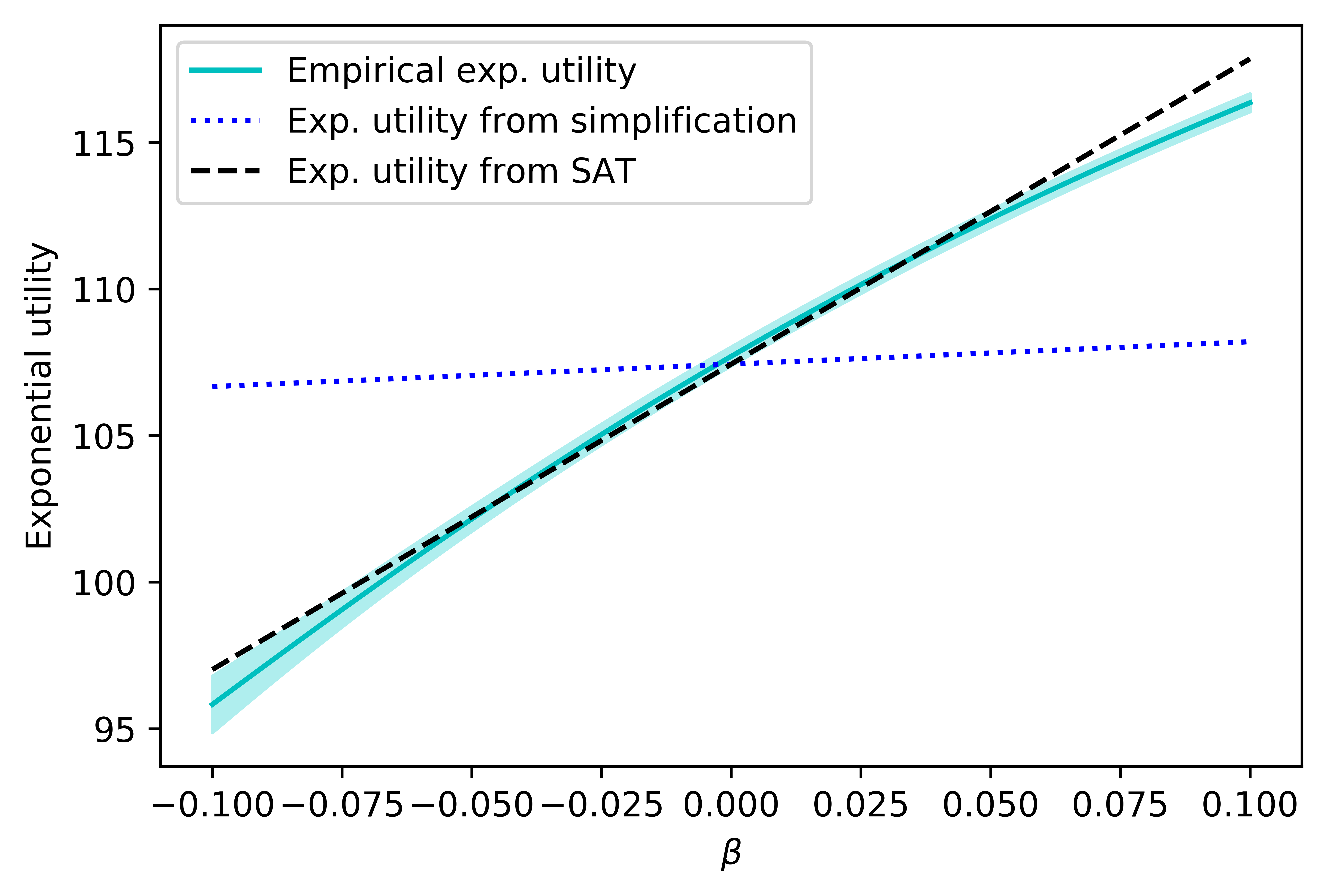} 
	\caption{The comparison among the empirical exponential utility risk, the estimated exponential utility risk from reward simplification, and the estimated exponential utility risk from SAT along the risk parameter $ \beta \in (-0.1, 0.1) $.} 
	\label{utility_1}
\end{figure}

\textit{Exponential utility risk estimation}: 
The empirical estimation of exponential utility risk with a risk parameter $ \beta $ is calculated by
\[
\hat{\Psi}_U (\beta) = \sum_{i=1}^{L}\Psi_{U,i}(\beta) / L,
\]
in which 
\begin{equation} \label{eq:UtExp}
	\Psi_{U,i}(\beta) = \beta^{-1} \log [\sum_{t=1}^{M} \exp (\beta \phi_{i,t})/M],
\end{equation}
where $ \phi_{i,t} $ is an outcome of a simulation in group $ i $, and $ t \in \{1, \cdots, M\} $.
The comparison among the empirical exponential utility risk, the estimated exponential utility risk from reward simplification, and the estimated exponential utility risk from SAT along the risk parameter $ \beta \in (-0.1, 0.1) $  is shown in Fig.~\ref{utility_1}.
It is worth noting that, the utility value at $ \beta = 0 $ is set by the average of the two adjacent values, since as the denominator in the Eq.~\ref{eq:UtExp}, $ \beta $ cannot be zero.
Based on the observation, we can see that when $ \beta \in (-0.1, 0.1) $, the estimated risk from SAT is close to the empirical one, but its counterpart from reward simplification is not.
That is again because the return variance is preserved by the SAT.
It is also found that, for the case $ \beta \in (-3, 3) $ in Fig.~\ref{utility_3}, the estimated risk from SAT goes far from the empirical one.
That is because there exists a term $ \mathcal{O}(\beta^2) $ in Eq.~\ref{eq:utility}, and when $ \beta $ deviates too much from zero, this term brings a large error.
Therefore, the estimation from SAT works only when $ \beta$ is close to zero, and its counterpart from reward simplification has a large error in both cases.

Thus far, we have shown that, the reward simplification changes the two risk estimations.
There are a number of methods for MDPs and Markov processes with rewards which are not stochastic and transition-based, such as~\cite{xia2018mean} on mean-variance risk and~\cite{borkar2002q,shen2014risk} on exponential utility risk. 
We believe that, when people apply these methods for practical problems with stochastic and transition-based rewards, they should revisit the methods with the SAT instead of the reward simplification.
This concern refers to the case when randomized policy is involved as well.

\begin{figure}[t] 
	\centering
	\includegraphics[width=\columnwidth]{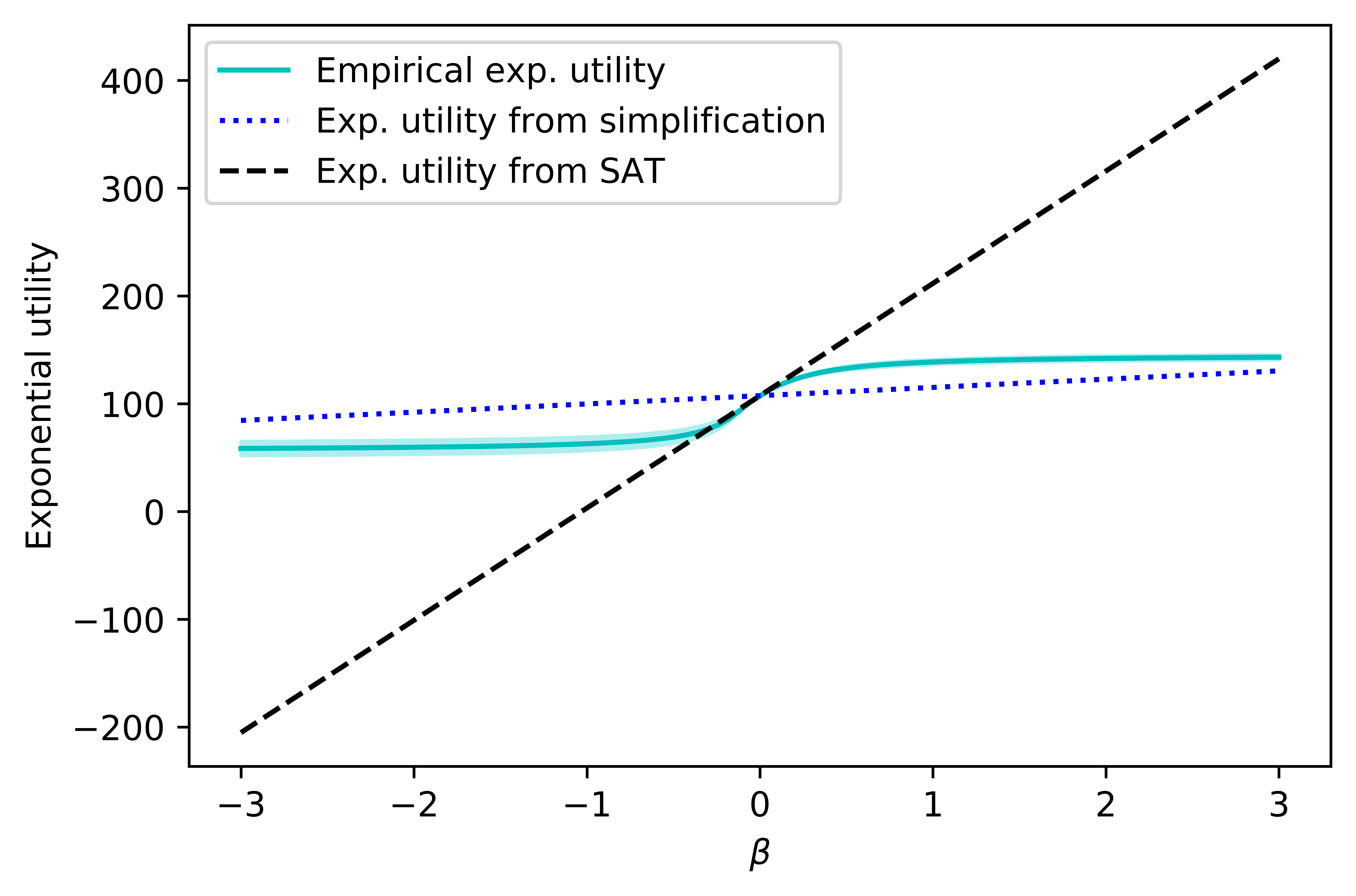} 
	\caption{The comparison among the empirical exponential utility risk, the estimated exponential utility risk from reward simplification, and the estimated exponential utility risk from SAT along the risk parameter $ \beta \in (-3, 3) $.} 
	\label{utility_3}
\end{figure}
\section{CONCLUSION AND FUTURE RESEARCH}
In this paper, we estimate the mean-variance risk and exponential risk with the return variance in a toy example.
With the aid of the SAT, the variance formula is extended for MDPs with a stochastic transition-based reward and a randomize policy.
Moreover, a definition of isotopic states is proposed in order to lump states and shrink the enlarged state space.
The numerical experiment illustrates the SAT with a state lumping, and validates the risk estimation methods.

One future work is to deal with the augmented state space.
Though the isotopic class can be lumped into one state under conditions, the space is still very large in most cases.
Another is to evaluate or estimate a generic risk measure iteratively.
Currently, only a few risks can be evaluated in iterations.
The third is 
to extend the methods for Markov processes with continuous time horizons and spaces in a risk-sensitive case. 




%

%


\bibliographystyle{ieeetr}
\bibliography{root.bib}


\end{document}